\pdfoutput=1

\documentclass[11pt]{article}

\usepackage[]{acl}

\usepackage{times}
\usepackage{latexsym}

\usepackage[T1]{fontenc}

\usepackage{inconsolata}
\usepackage{times}
\usepackage{latexsym}
\usepackage{amsfonts}
\usepackage{amsmath}
\usepackage{amssymb}
\usepackage{amsthm}
\usepackage{multicol}
\usepackage{multirow}
\usepackage{xspace}
\usepackage{booktabs}
\usepackage{bbding}
\usepackage{array}
\usepackage{threeparttable}
\usepackage{tcolorbox}
\usepackage{tabularx}
\usepackage{enumitem}
\usepackage[linesnumbered,ruled,vlined]{algorithm2e}
\usepackage{xcolor,colortbl}
\newtheorem{theorem}{Theorem}[section]
\usepackage{color}
\usepackage{setspace}
\usepackage{makecell}
\usepackage{listings}
\usepackage{xltabular}
\usepackage{cleveref}
\crefformat{section}{\S#2#1#3}
\crefformat{subsection}{\S#2#1#3}
\crefformat{subsubsection}{\S#2#1#3}
\crefrangeformat{section}{\S\S#3#1#4 to~#5#2#6}
\crefmultiformat{section}{\S\S#2#1#3}{ and~#2#1#3}{, #2#1#3}{ and~#2#1#3}
\crefmultiformat{subsection}{\S\S#2#1#3}{ and~#2#1#3}{, #2#1#3}{ and~#2#1#3}
\Crefformat{figure}{#2Fig.~#1#3}
\Crefmultiformat{figure}{Figs.~#2#1#3}{ and~#2#1#3}{, #2#1#3}{ and~#2#1#3}
\Crefformat{table}{#2Tab.~#1#3}
\Crefmultiformat{table}{Tabs.~#2#1#3}{ and~#2#1#3}{, #2#1#3}{ and~#2#1#3}
\Crefformat{appendix}{Appx.~\S#2#1#3}
\crefmultiformat{appendix}{Appx.~\S#2#1#3}{ and~#2#1#3}{, #2#1#3}{ and~#2#1#3}
\crefformat{algorithm}{Alg.~#2#1#3}
\Crefformat{equation}{Eq.~#2#1#3}

\newcommand{\xhdr}[1]{\vspace{0.3em}\noindent{{\bf #1.}}}

\title{Speculative Contrastive Decoding}

\author{
Hongyi Yuan$^{12}$\thanks{$^*$Work done during internship at Alibaba Inc.} , Keming Lu$^{2}$, Fei Huang$^{2}$, Zheng Yuan$^{2}$, Chang Zhou$^{2}$
\\
$^1$Tsinghua University, $^2$Alibaba Inc. \\
\texttt{yuanhy20@mails.tsinghua.edu.cn} \\
\texttt{\{lukeming.lkm,feihu.hf\}@alibaba-inc.com}\\
\texttt{\{yuanzheng.yuanzhen,ericzhou.zc\}@alibaba-inc.com}\\
}

\begin{document}
\maketitle
\begin{abstract}

Large language models~(LLMs) exhibit exceptional performance in language tasks, yet their auto-regressive inference is limited due to high computational requirements and is sub-optimal due to the exposure bias. 
Inspired by speculative decoding and contrastive decoding, we introduce Speculative Contrastive Decoding~(SCD), a straightforward yet powerful decoding approach that leverages predictions from smaller language models~(LMs) to achieve both decoding acceleration and quality improvement. 
Extensive evaluations and analyses on four diverse language tasks demonstrate the effectiveness of SCD, showing that decoding efficiency and quality can compatibly benefit from one smaller LM.

\end{abstract}

\section{Introduction}
Large language models~(LLMs) have advanced the versatility and proficiency in approaching real-world natural language tasks such as general instruction following \citep{instructgpt,alpaca,instag} and reasoning \citep{cobbe2021training,cot,rft}.
Most existing LLMs (\citet{Brown2020LanguageMA,llama2,qwen},\textit{inter alia}) are built on decoder-only Transformers.
Due to the auto-regressive nature during inference, the runtime of decoding inference can be excessive on general computation infrastructure, and the generation quality can be sub-optimal due to the exposure bias \citep{arora-etal-2022-exposure}.
Improving decoding inference has been the spotlight of the research community in language generation \citep{vijayakumar2018diverse,Holtzman2020The,su2022contrastive}.

As for decoding acceleration, one prominent method named speculative decoding~\cite{Leviathan2022FastIF,chen2023accelerating} has been proposed and leverages relatively smaller language models~(LMs) to predict several successive token generations of target LLMs. 
The LLMs only require one-time forward computation for checking the validity of predictions from the smaller LMs. 
The decoding method maintains the target LLMs' token distributions and accelerates more when smaller LMs can accurately predict the potential target LLMs' generations.


As for the generation quality, contrastive decoding has been recently proposed \citep{li-etal-2023-contrastive}. 
Contrastive decoding assumes that conjugated smaller LMs may present higher systematic tendencies to generate erroneous tokens than the larger ones, and the method seeks to eliminate such systematic error by contrasting the token distribution between smaller LMs and larger LMs. 
From either inference acceleration or quality improvement, these works have demonstrated a promising direction by integrating smaller LMs during auto-regressive generation. 

Inspired by both speculative and contrastive decoding, we propose Speculative Contrastive Decoding~(SCD), which exploits a single smaller LM for decoding improvement in speed and quality en bloc. 
Comprehensive evaluations of four diverse tasks show that SCD can achieve similar acceleration factors of speculative decoding while maintaining the quality improvement from contrastive decoding. 
By further analyzing the token distributions of the smaller and larger LMs in SCD, we show the inherent compatibility of decoding acceleration and quality improvement.
The contributions of this paper can be summarized as follows:
\begin{itemize}[leftmargin=1em]
    \vspace{-0.4em}
    \setlength\itemsep{-0.3em}
    \item We propose Speculative Contrastive Decoding for efficacious LLM inference.
    \item Comprehensive experiments and analysis illustrate the compatibility of speculative and contrastive decoding on 4 diverse tasks. 
\end{itemize}

\section{Related Works}

In terms of inference acceleration, recent research has been devoted to developing various efficient decoding methods \citep{yao2022zeroquant,kwon2023efficient,medusa}. Speculative decoding \citet{Leviathan2022FastIF,chen2023accelerating,Kim2023SpeculativeDW} is one of these recent works and utilizes smaller models for acceleration. 
\citet{miao2023specinfer,spector2023accelerating} propose to organize predictions from small LMs into tree structures to accelerate speculative decoding further.
In terms of inference quality, rich research has been suggested \citep{vijayakumar2018diverse,Holtzman2020The,su2022contrastive,su2022empirical,finlayson2023closing} and contrastive decoding achieves better decoding qualities by similarly integrating smaller LMs and devise contrastive token distributions~\citep{li-etal-2023-contrastive,obrien2023contrastive}. 
It can further be adjusted to other variants such as the token distribution contrasting between model layers~\cite{chuang2023dola} or different inputs~\cite{yona2023surfacing}.
SCD draws inspiration from these works and benefits both decoding speed and quality by incorporating smaller LMs into generation.


\section{Preliminaries}

We follow the terminology in \citet{li-etal-2023-contrastive}, and term the target larger LMs as the expert LMs while the smaller LMs as the amateur LMs denoted as $\mathcal{M}_e$ and $\mathcal{M}_a$ respectively. 

\subsection{Contrastive Decoding}\label{sec:cd}

The intrinsic rationale of contrastive decoding~(CD) is that amateur LMs have stronger systematic undesirable tendencies to produce undesirable patterns (e.g., hallucination) than expert LMs. 
By contrasting the token distributions between expert and amateur LMs, such tendencies can be alleviated.
There have been successively proposed two versions of contrastive decoding by \citet{li-etal-2023-contrastive} and \citet{obrien2023contrastive}, which we term as \textit{Original} contrastive decoding and \textit{Improved} contrastive decoding.
The final contrastive logit scores for the original contrastive decoding $s_{\text{ori}}(x_i|x_{<i})$ and the improved contrastive decoding $s_{\text{imp}}(x_i|x_{<i})$ are respectively:

\vspace{-12pt}
\small{
\begin{align*}
   & s_{\text{ori}}(x_i|x_{<i}) = \\
   & \left\{
   \begin{array}{lc}
             \log P_{\mathcal{M}_e}(x_i|x_{<i}) - \log P_{\mathcal{M}_a}(x_i|x_{<i}), & x_i\in\mathcal{V}^\alpha_{\text{ori},i}  \\
             -\infty, & x_i\notin\mathcal{V}^\alpha_{\text{ori},i} \\
    \end{array}
    \right.\\
    & s_{\text{imp}}(x_i|x_{<i}) = \\
   & \left\{
   \begin{array}{lc}
             (1+\beta) Y_{\mathcal{M}_e}(x_i|x_{<i}) - \beta Y_{\mathcal{M}_a}(x_i|x_{<i}), & x_i\in\mathcal{V}^\alpha_{\text{imp},i}  \\
             -\infty, & x_i\notin\mathcal{V}^\alpha_{\text{imp},i} \\
    \end{array}
    \right.
\end{align*}
}\normalsize

where $P_{\cdot}$ and $Y_{\cdot}$ are respectively the token probability and logit generated from LMs. $\mathcal{V}^\alpha_{\cdot,i}$ denotes the adaptive plausibility constraint that dynamically restricts the logits from producing the erroneous modes. The adaptive plausibility constraints are calculated as

\vspace{-12pt}
\small{
\begin{align*}
    &\mathcal{V}^\alpha_{\text{ori},i} = \left\{w|P_{\mathcal{M}_e}(w|x_{<i}) > \alpha \max_{w\in \mathcal{V}} P_{\mathcal{M}_e}(w|x_{<i})\right\}, \\
    &\mathcal{V}^\alpha_{\text{imp},i} = \left\{w|Y_{\mathcal{M}_e}(w|x_{<i}) > \log\alpha + \max_{w\in \mathcal{V}} Y_{\mathcal{M}_e}(w|x_{<i})\right\}. 
\end{align*}
}\normalsize

A token is generated from the contrastive token distribution $P^\tau_{n}(x_i) = \operatorname{softmax}_\tau\left(s_{n}(x_i|x_{<i})\right)$, $n\in\{\text{ori},\text{imp}\}$, where $\tau$ represents the softmax temperature that determines the smoothness of the contrastive token distribution. 





\subsection{Speculative Decoding}

Instead of requiring one forward computation of $\mathcal{M}_e$ for each token in vanilla decoding, speculative decoding (SD) utilizes $\mathcal{M}_a$ to primarily generate $\gamma$ tokens at each iteration then $\mathcal{M}_e$ makes one forward computation to check the validity of the $\gamma$ tokens. 
If $\mathcal{M}_e$ accepts all the $\gamma$ tokens, it finishes the iteration with an additional generated token, resulting in $\gamma + 1$ tokens generated. 
Otherwise, if $\mathcal{M}_e$ rejects a token at $r$, the token is re-sampled according to $\mathcal{M}_e$ to substitute the rejected token; hence the iteration finishes with $r$ tokens generated. 
With only one-time forward computation of $\mathcal{M}_e$, multiple tokens are generated at each iteration. When the ratio between the runtime required of $\mathcal{M}_a$ and $\mathcal{M}_e$ (the cost coefficient $c$, \citet{Leviathan2022FastIF}) is low and the token acceptance rate is high, there will present a notable acceleration.



\begin{algorithm}[t]
\small
\caption{\small{Speculative Contrastive Decoding}}\label{algo}
\KwData{$\mathcal{M}_e$, $\mathcal{M}_a$, input prefix $x_{\text{inp}}$}
\KwResult{$[x_{\text{inp}}, x_1,..,x_{\text{k}}]$}
\For{$i$ from $1$ to $\gamma$}{
$x_i \sim P_{\mathcal{M}_a}(x_i) = \mathcal{M}_a(x_i|x_{\text{inp}},x_{<i})$\;
}
$P_{\mathcal{M}_e}(x_1),..,P_{\mathcal{M}_e}(x_{\gamma+1}) = \mathcal{M}_e(x_1,..,x_\gamma|x_{\text{inp}})$\;
\textcolor{blue}{Calculate $P_n(x_1),..,P_n(x_\gamma)$ following Section \Cref{sec:cd}\;}
$r_1,..,r_\gamma$  i.i.d sampled from  $ \text{Uniform}(0, 1)$\;
\textcolor{blue}{$k = \min\left( \{i|r_i >\frac{P_{n}(x_i)}{P_{\mathcal{M}_a}(x_i)}\}\cup \{\gamma+1\}\right)$\;}
\uIf{$k \leq \gamma$}{
    $P_{k}(x_{k}) = \operatorname{norm}(\max(0, P_{n}(x_k)- P_{\mathcal{M}_a}(x_k))$\;
    Resample $x_k \sim P_{k}(x_{k})$\;
  }
  \Else{
    \textcolor{blue}{$P_{\mathcal{M}_a}(x_{\gamma+1}) = \mathcal{M}_a(x_{\gamma+1}|x_{\text{inp}},x_{1},..,x_\gamma)$\;}
    \textcolor{blue}{Calculate $P_n(x_{\gamma+1})$ following Section \Cref{sec:cd}\;}
    $x_{\gamma+1} \sim P_n(x_{\gamma+1})$\;
  }
\end{algorithm}

\section{Speculative Contrastive Decoding}

Speculative decoding leverages smaller $\mathcal{M}_a$ only for generation acceleration, while not making the best of the token distributions from $\mathcal{M}_a$. 
It is natural to simultaneously apply the contrastive token distribution, and with negligible computational overhead, the generation quality and efficiency can benefit from integrating speculative and contrastive decoding. 
Therefore, we propose Speculative Contrastive Decoding (SCD).

Concretely, at each iteration, $\gamma$ tokens are generated from the amateur model $\mathcal{M}_a$. 
When checking the validity of the tokens, the target distribution becomes $P^\tau_{n}, n\in\{\text{ori},\text{imp}\}$ from contrastive distribution instead of $P_{\mathcal{M}_e}$ in speculative decoding. 
For a token $x$ in the $\mathcal{M}_a$-generated tokens, it is rejected with probability $1-\frac{P^\tau_{n}(x)}{P_{\mathcal{M}_a}(x)}$ and then a new token in place of $x$ is re-sampled from $\operatorname{norm}(\max(0, P^\tau_{n}(x)- P_{\mathcal{M}_a}(x))$, where $\operatorname{norm}\left(f(x)\right)=f(x)/\sum_xf(x), \text{s.t.} f(x)\ge0$. 
If all the $\mathcal{M}_a$-generated tokens are accepted, then an additional token is sampled from $P_{n}^\tau$.

The sampling procedure of SCD is similar to the original speculative decoding in \citet{Leviathan2022FastIF,chen2023accelerating}. 
However, it is worth noticing that in our SCD, when all the $\mathcal{M}_a$-generated tokens are accepted, we require an additional forward computation from $\mathcal{M}_a$ to acquire its last token logit for calculating the contrastive distribution $P_n^\tau$ at that iteration, while in speculative decoding, the additional token is sampled directly from $\mathcal{M}_e$.
This computational overhead is negligible when $c$ is small. 
We detailed the algorithm of our SCD in Algorithm \Cref{algo}. The difference from the original speculative decoding is highlighted in \textcolor{blue}{blue}.

\begin{figure*}[t]
    \centering
    \includegraphics[width=\linewidth]{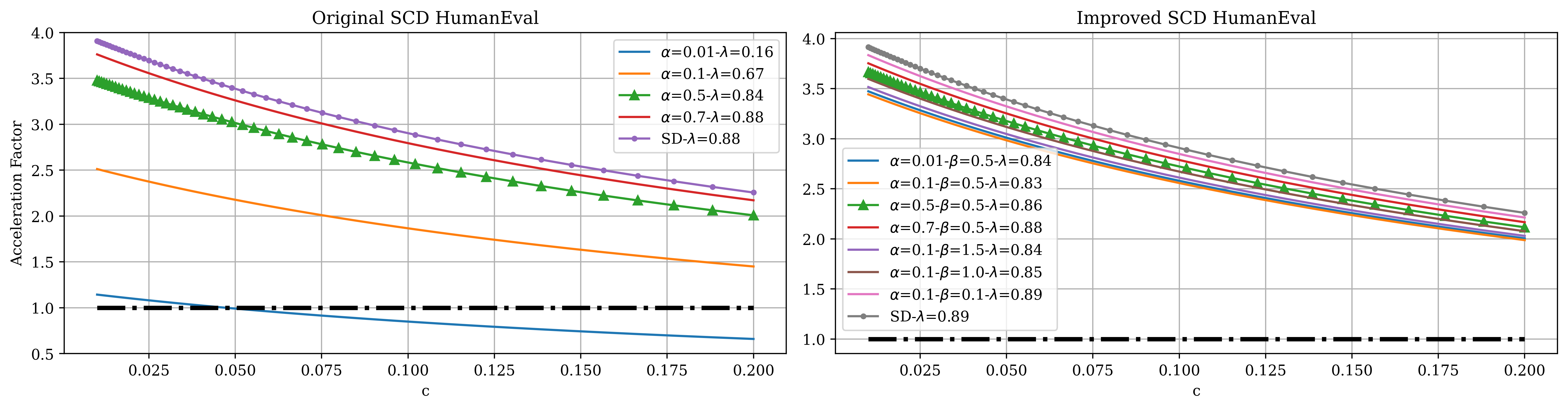}\vspace{-7pt}
   \caption{
    Hyper-parameter analysis on expected acceleration factors regarding empirical acceptance rate $\lambda$.
    The best hyper-parameter settings as in \Cref{tab:main_table} are the lines marked with triangles.
    }\vspace{-7pt}
    \label{fig:analysis_hyperparam_humaneval}
\end{figure*}

\begin{figure*}[t]
    \centering
    \includegraphics[width=\linewidth]{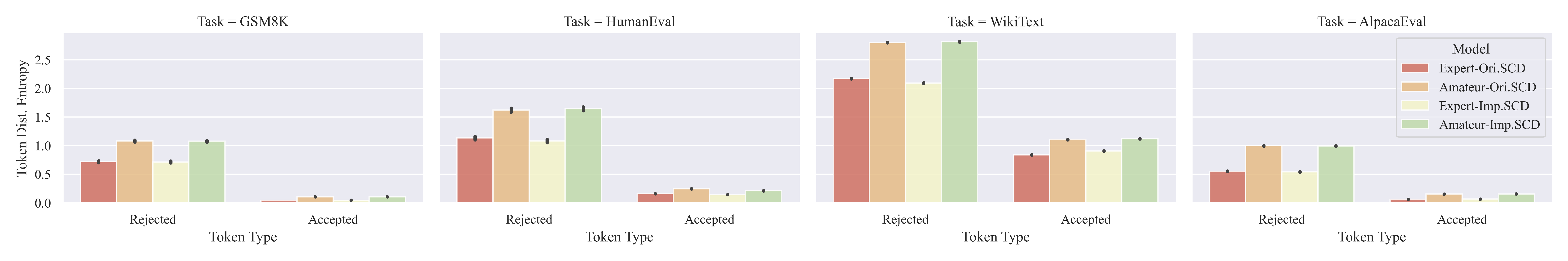}\vspace{-7pt}
    \caption{
    The averaged token distribution entropy with error bars of rejected and accepted tokens in SCD.
    }\vspace{-7pt}
    \label{fig:entropy_analysis}
\end{figure*}

\section{Experiment}\label{sec:experiment}

\xhdr{Experiment Setting}
We evaluate SCD and other baselines on four benchmarks: \textbf{WikiText}~\cite{merity2016pointer}, \textbf{HumanEval}~\cite{chen2021evaluating}, \textbf{AlpacaEval}~\cite{alpaca_eval}, and \textbf{GSM8k}~\cite{cobbe2021training}.
The four benchmarks span diverse language tasks of open-ended generation, code generation, human alignment, and mathematical reasoning respectively.
For WikiText, we use the pre-trained Llama2\textsubscript{7B} and Llama2\textsubscript{70B} \citep{llama2} as $\mathcal{M}_a$ and $\mathcal{M}_e$ and follow \citet{li-etal-2023-contrastive} to use diversity, MAUVE \citep{pillutla2021mauve} and coherence as evaluation metrics. 
For HumanEval, we use the pre-trained Llama2\textsubscript{7B} and Llama2\textsubscript{70B} and assess the 1-round pass rate. 
For AlpacaEval, we use human-aligned Llama2chat\textsubscript{7B} and Llama2chat\textsubscript{70B} and report win-rates over \textit{text-davinci-003} judged by GPT-4.
For GSM8k, we use fine-tuned Llama2\textsubscript{7B} and Llama2\textsubscript{70B} on its training set and report the accuracy of the test-set results.
We set $\gamma=4$ across all experiments and set the temperature $\tau$ to 0.7 for WikiText and AlpacaEval and 0.001 for GSM8k and HumanEval.
We leave the detailed experiment settings to \Cref{app:exp_set}.

\begin{table}[t]
\centering
\small
\setlength{\tabcolsep}{0.5mm}{
\begin{tabular}{lcccccc}
\toprule
 & \multicolumn{3}{c}{\textbf{WikiText}}& \textbf{A.Eval} & \textbf{GSM8k} & \textbf{H.Eval}  \\
 &  Div. & MAU. &  Coh.& Score & Acc. &Pass@1 \\
\midrule
$\mathcal{M}_a$ &$0.69_{.00}$ &$0.88_{.01}$ &$0.76_{.00}$ &$88.79_{1.1}$ &$41.77_{.00}$& $11.59_{.0}$  \\
$\mathcal{M}_e$ &$0.75_{.00}$ &$0.88_{.01}$ &$0.75_{.00}$ &$94.66_{.79}$ &$64.19_{.04}$ & $28.66_{.0}$  \\
SD&$0.75_{.00}$ &$0.90_{.01}$ &$0.75_{.01}$ &$94.28_{.83}$ &$64.27_{.07}$ &$28.66_{.0}$  \\
\midrule
CD\textsubscript{ori} &$0.91_{.00}$ &$0.95_{.00}$ &$0.73_{.00}$ &$94.56_{.82}$ &$64.42_{.03}$ &$37.20_{.0}$ \\
SCD\textsubscript{ori} &$0.91_{.00}$ &$0.94_{.00}$ &$0.72_{.01}$ &$94.91_{.78}$ &$64.44_{.06}$ & $37.20_{.0}$  \\
E.A.\textsubscript{ori} &\multicolumn{3}{c}{$\times1.78$}&$\times2.92$&$\times3.32$&$\times3.01$\\
\midrule
CD\textsubscript{imp} &$0.73_{.01}$ &$0.90_{.01}$ &$0.74_{.00}$ &$94.78_{.79}$ &$64.91_{.01}$ & $33.54_{.0}$ \\
SCD\textsubscript{imp} &$0.73_{.00}$ &$0.91_{.01}$ &$0.74_{.00}$ &$95.03_{.77}$ &$64.90_{.02}$ & $33.54_{.0}$  \\
E.A.\textsubscript{imp}&\multicolumn{3}{c}{$\times2.10$}&$\times2.95$ &$\times3.32$&$\times3.18$ \\
\bottomrule
\end{tabular}}
\caption{
Main results of SCD.
H.Eval, and A.Eval are shorts for HumanEval and AlpacaEval.
MAU. and Coh. are shorts for MAUVE and coherence. 
E.A. presents the expected acceleration under $c=0.05$.
The standard errors under 3 repetitions for each result are marked in subscripts.
The best choices of $\alpha$ and $\beta$ for (S)CD are left to \Cref{app:best-hyper}.
}\vspace{-10pt}
\label{tab:main_table}
\end{table}

\xhdr{Quality Results}
As shown in \Cref{tab:main_table}, original and improved SCD and CD demonstrate significant improvement over $\mathcal{M}_e$ in GSM8k and HumanEval. 
On WikiText, only original CD and SCD outperform $\mathcal{M}_e$ in terms of diversity with $+0.16$ and MAUVE with $+0.06$. 
There is no obvious improvement in Coherence.
On AlpacaEval, although both versions of SCD and CD show better results than $\mathcal{M}_e$, such improvement is not significant due to the high variance of GPT4-as-a-judge.
We can see that different versions of SCD suggest different levels of improvement. Original SCD performs better on WikiText and HumanEval while inferior on GSM8k to improved SCD. 
Results across four benchmarks show SCD can benefit various LLMs on diverse language tasks, maintaining the same generation quality improvement as CD.

\xhdr{Acceleration}
To demonstrate the inference acceleration of SCD, we primarily provide the expected acceleration factor of SCD theoretically with respect to the number of $\mathcal{M}_a$ token predictions per iteration $\gamma$, the acceptance rate $\lambda$, and the cost coefficient $c$, which proof is left to \Cref{app:proof}.
\begin{theorem}\label{theo:main}
The expected acceleration factor in decoding runtime is $\frac{1-\lambda^{\gamma+1}}{(1-\lambda)(1 + c\gamma + c\lambda^\gamma)}$.
\end{theorem}
In \Cref{tab:main_table}, consistent acceleration is presented across different benchmarks.
We further visualize the expected acceleration factor of SCD in \Cref{fig:analysis_hyperparam_humaneval} according to the empirical acceptance rates $\lambda$ in HumanEval with different hyper-parameter settings. 
According to \Cref{theo:main}, the acceleration factors are depicted against the cost coefficient $c$, which is usually of small values representing the ratio of runtime required of $\mathcal{M}_a$ and $\mathcal{M}_e$ and depends on the infrastructures (e.g., GPU) that serve the LLMs. We can see that the acceptance rates hence the corresponding acceleration factors of original SCD are more sensitive to hyper-parameters compared to improved SCD. With proper hyper-parameters, SCD can achieve similar acceleration to the speculative decoding (dotted lines), which indicates the negligible speed trade-off to incorporate the contrastive token distributions. Results on GSM8k are listed in \Cref{app:addition_result} presenting similar patterns.

\section{Analysis}\label{sec:analysis}
\xhdr{Compatibility}
Results presented in \Cref{sec:experiment} show SCD can combine the benefits of CD and SD.
We delve deep into the reasons for such compatibility.
We calculate the average entropy of token probabilities from $\mathcal{M}_a$ and $\mathcal{M}_e$ regarding the accepted and rejected tokens in SCD.
As shown in \Cref{fig:entropy_analysis}, token distribution entropy from both $\mathcal{M}_a$ and $\mathcal{M}_e$ of accepted tokens is significantly higher than that of rejected tokens.
The phenomenon suggests SCD enjoys acceleration from accepting easy tokens of lower entropy while benefiting from contrastive token distribution by rejecting hard tokens of higher entropy. 
We also present a case study from GSM8k in \Cref{app:case} to demonstrate such compatibility.

\begin{figure}[t]
    \centering
    \includegraphics[width=\linewidth]{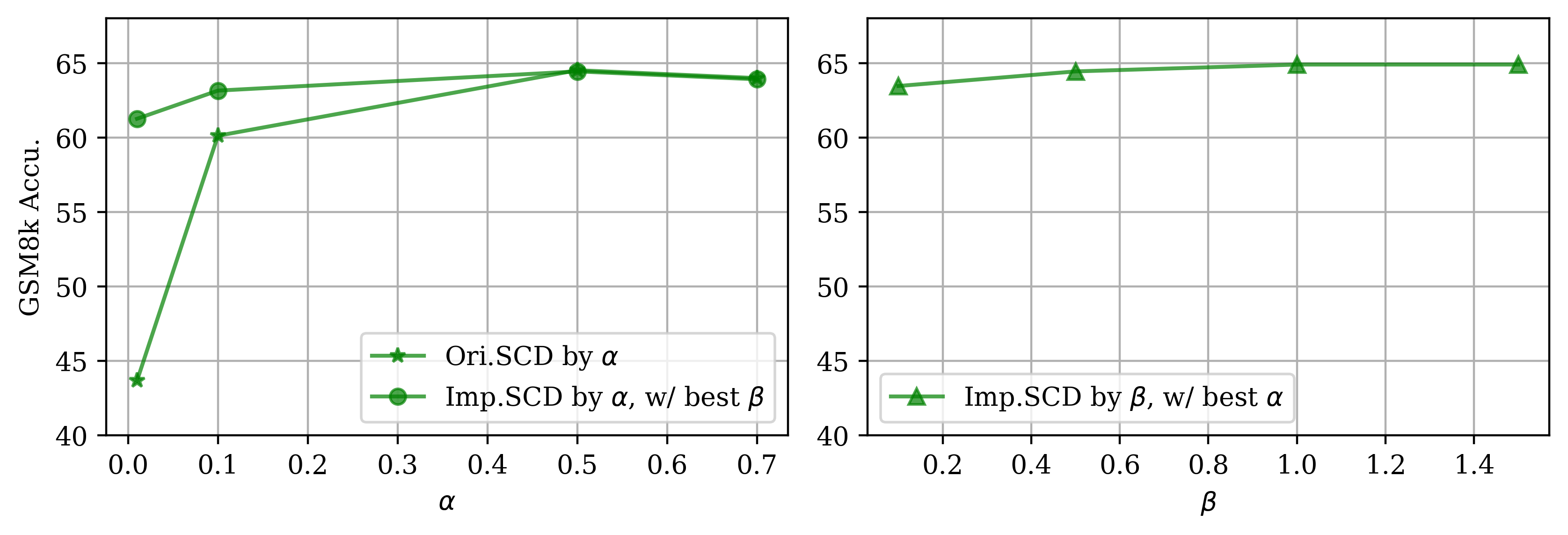}\vspace{-10pt}
    \caption{ Performance sensitivity regarding $\alpha$ and $\beta$.
    }
    \vspace{-10pt}
    \label{fig:analysis_performance}
\end{figure}

\xhdr{Sensitivity} 
Through \Cref{fig:analysis_performance}, we show how performances fluctuate with respect to the hyper-parameter $\alpha$ and $\beta$. We can see that improved SCD is less sensitive to both $\alpha$ and $\beta$ on GSM8k compared to the original SCD. 
This is possibly due to the better flexibility of manipulating logits than probabilities. 
Results on HumanEval are listed in \Cref{app:addition_result} presenting similar phenomenons.



\section{Conclusion}

In this paper, we propose speculative contrastive decoding, a decoding strategy that naturally integrates small amateur LMs for inference acceleration and quality improvement of LLMs.
Extensive experiments show the effectiveness of SCD and our delve-deep analysis also explains the compatibility through the scope of token distribution entropy. 
Our method can be easily deployed to improve the real-world serving of LLMs.

\section*{Limitation}

In our experiments, we provide the expected acceleration factors of SCD on four benchmarks calculated according to the empirical token acceptance rates $\lambda$ and selected cost coefficients $c$.
The empirical acceleration factor is highly correlated to the actual infrastructures that serve both the larger LMs and the smaller LMs.
To compensate for this demonstration limitation and better demonstrate the acceleration performance, we visualize the expected acceleration factor by spanning across a range of $c$ in \Cref{fig:analysis_hyperparam_humaneval}.
This is a common limitation of deploying speculative decoding in the real-world LLM serving.
For example, the runtime of switching between the forward computation of $\mathcal{M}_a$ and $\mathcal{M}_e$ would be non-negligible without properly optimized infrastructures, causing a relatively large $c$ hence potentially resulting in deceleration even with high acceptance rates.

\section*{Broader Impact}

Although LLMs have demonstrated exceptional performance and been helpful real-world assistants recently, the massive computational demands of LLMs forbid most users including potential researchers from local deployments, who generally alter to use APIs from LLM servings.
Therefore, effective methods, including our SCD, to improve the speed and quality from the perspective of decoding inference have much potential to advance LLM-based services.

\bibliography{anthology,custom}
\bibliographystyle{acl_natbib}

\appendix

\section{Experiment Details}\label{app:exp_set}

\subsection{Benchmark Details}

(1) \textbf{WikiText}~\cite{merity2016pointer} contains articles from Wikipedia. We follow the pre-processing scripts from \citet{li-etal-2023-contrastive} and result in 1,733 samples. The generation starts with the first 32 tokens as prompts, and the max generation length is set to 256. We report diversity, MAUVE \citep{pillutla2021mauve}, and coherence as metrics, following \citet{li-etal-2023-contrastive}. 

Diversity metrics assess the unique multi-grams in the completion generated from the LMs. Higher diversity scores indicate better lexical diversity in the completion. The diversity is calculated according to:
\begin{align*}
    \text{Div.} = \prod_{n=2}^4 \frac{|\operatorname{Set}(\text{n-grams})|}{|\text{n-grams}|}.
\end{align*}
MAUVE is a metric proposed by \citet{pillutla2021mauve}, which is empirically suggested to have better agreement with human annotations \citep{gao-wan-2022-dialsummeval}.
Coherence evaluates the semantic correlation between the input prefix and the output generation via the similarity of embeddings. We use the sentence embeddings following SimCSE \citep{gao-etal-2021-simcse} and the coherence score is calculated as:
\begin{align*}
    \frac{\text{emb}(x_{\text{prefix}})\cdot\text{emb}(x_{\text{gen}})}{\|\text{emb}(x_{\text{prefix}})\|\|\text{emb}(x_{\text{gen}})\|}.
\end{align*}

(2) \textbf{GSM8k}~\cite{cobbe2021training} contains training and evaluation sets of grade mathematical reasoning problems. We first fine-tune the Llama2\textsubscript{7B} and Llama2\textsubscript{70B} by 3 epochs to produce the amateur and expert LMs.
We report the final accuracy of the test sets.

(3) \textbf{HumanEval}~\cite{chen2021evaluating} measures coding correctness for synthesizing programs from 164 doc-strings. We report the 1-round pass rate~(Pass@1).

(4) \textbf{AlpacaEval}~\cite{alpaca_eval} contains 805 samples from various evaluation sets to evaluate the alignment abilities of LLMs by comparing evaluated models with \textit{text-davinci-003}. We report the win rate judged by GPT-4.

\subsection{Configuration Details}
We use Llama2\textsubscript{7B} as the amateur model while Llama2\textsubscript{70B} as the expert model on WikiText and HumanEval benchmarks to evaluate how SCD performs with pre-trained models.
Then, we fine-tune Llama2\textsubscript{7B} and Llama2\textsubscript{70B} on the GSM8k training set to evaluate the SCD performance with supervised fine-tuning models on the mathematical reasoning task.
We also apply Llama2chat\textsubscript{7B} and Llama2chat\textsubscript{70B} on AlpacaEval to assess LLMs for human alignment using SCD.
We set the softmax temperature consistent to 0.7 on WikiText and AlpacaEval while 0.001 on other benchmarks.
In SCD and SD, we always set the prediction temperature from the amateur LMs to 1.0 for fair comparison.
All experiments are conducted on 2 A100 80G GPUs with KV cache implementation.

\subsection{Hyper-parameter Details}\label{app:best-hyper}
We conduct grid searches regarding $\alpha$ and $\beta$ for the best performance of CD and SCD. The best hyper-parameter settings for the results in \Cref{tab:main_table} are listed in \Cref{apptab:hypersetting}.

\begin{table*}[t]
    \centering
    \begin{tabular}{l|cccccccc}
    \toprule
    &\multicolumn{2}{c}{WikiText}& \multicolumn{2}{c}{AlpacaEval} &\multicolumn{2}{c}{GSM8k} &\multicolumn{2}{c}{HumanEval}\\
    &$\alpha$&$\beta$ &$\alpha$&$\beta$ &$\alpha$&$\beta$ &$\alpha$&$\beta$ \\
    \midrule
      CD\textsubscript{ori} & 0.1&- & 0.5&-& 0.5&-& 0.5&- \\
      SCD\textsubscript{ori} & 0.1&- & 0.5&-& 0.5&-& 0.5&- \\
      CD\textsubscript{ori} & 0.1&0.5 & 0.5&0.5& 0.5&1.0& 0.5&0.5 \\
      SCD\textsubscript{ori} & 0.1&0.5 & 0.5&0.5& 0.5&1.0& 0.5&0.5 \\
      \bottomrule
    \end{tabular}
    \caption{The hyper-parameter settings for the results in \Cref{tab:main_table}}
    \label{apptab:hypersetting}
\end{table*}

\section{Proof of Theorem \Cref{theo:main}}\label{app:proof}

\begin{theorem}
The expected acceleration factor in decoding runtime is $\frac{1-\lambda^{\gamma+1}}{(1-\lambda)(1 + c\gamma + c\lambda^\gamma)}$.
\end{theorem}
\begin{proof}
Similar to Theorem 3.8 in \citet{Leviathan2022FastIF}, given the token acceptance rate $\lambda$ and the runtime per forward computation step for $\mathcal{M}_e$ and $\mathcal{M}_a$ are $T$ and $cT$. The total runtime required for each iteration is $T+c\gamma T+c\lambda^\gamma T$, where $\mathcal{M}_a$ requires $\gamma$ generation steps and possibly one additional step forward computation if all $\gamma$ tokens are accepted while $\mathcal{M}_a$ requires one forward computation for token validity checking.
Following Equation (1) in \citet{Leviathan2022FastIF}, the expected generated token number per iteration is $\frac{1-\lambda^{\gamma+1}}{1-\lambda}$.
Therefore, the expected runtime needed of SCD is $\frac{1-\lambda}{1-\lambda^{\gamma+1}}(T+c\gamma T+c\lambda^\gamma T)$, hence the expected acceleration factor is $\frac{1-\lambda^{\gamma+1}}{(1-\lambda)(1 + c\gamma + c\lambda^\gamma)}$.
\end{proof}

\section{Case Study}\label{app:case}

\begin{tcolorbox}[title=Case Study on GSM8k,colback=white]\label{box:case}
\small
Judy teaches 5 dance classes, every day, on the weekdays and 8 classes on Saturday.  If each class has 15 students and she charges \$15.00 per student, how much money does she make in 1 week?
\tcbsubtitle{Responses}
[\textcolor{red}{RED} and \textcolor{blue}{BLUE} tokens are generated by the expert LM after rejection and all acceptance, respectively. Other tokens are generated by the amateur LM.] \\\\
\textcolor{red}{She} teaches 5 \textcolor{blue}{classes} \textcolor{red}{every} weekday so that's \textcolor{blue}{5}*5 \textcolor{red}{=} <<5*5=25>>25 classes \\
She teaches \textcolor{red}{2}5 \textcolor{red}{week}day classes and \textcolor{blue}{8} \textcolor{red}{Saturday} classes for a total \textcolor{blue}{of} 25+\textcolor{blue}{8} = <<25\textcolor{blue}{+}8=33\textcolor{blue}{>>}33 classes \\
\textcolor{red}{Each class has} 15 students \textcolor{blue}{so} \textcolor{red}{she teach}es \textcolor{red}{1}5\textcolor{red}{*}33 = <<\textcolor{blue}{1}5*33\textcolor{blue}{=}495>>\textcolor{blue}{4}95 students \textcolor{red}{in} 1 week \\
\textcolor{blue}{She} charges \$15\textcolor{blue}{.}00 per student \textcolor{blue}{and} \textcolor{red}{she teach}es 49\textcolor{blue}{5} students \textcolor{red}{so she} makes 15\textcolor{blue}{*}495 = \textcolor{blue}{\$}<<15*\textcolor{blue}{4}95=7\textcolor{blue}{4}\textcolor{red}{2}5.00>>7,425.00 in 1 week
\end{tcolorbox}
In this case, we can see that the rejected and re-sampled tokens are usually the beginning of a sentence, numbers, operations, or named entities, which are generally informative tokens in the reasoning chain of thoughts. This also indicates that quality improvement originates from re-sampling informative tokens by contrastive token distribution while the acceleration comes from speculative prediction of the amateur LMs.

\section{Additional Results}\label{app:addition_result}

\begin{figure*}[t]
    \centering
    \includegraphics[width=\linewidth]{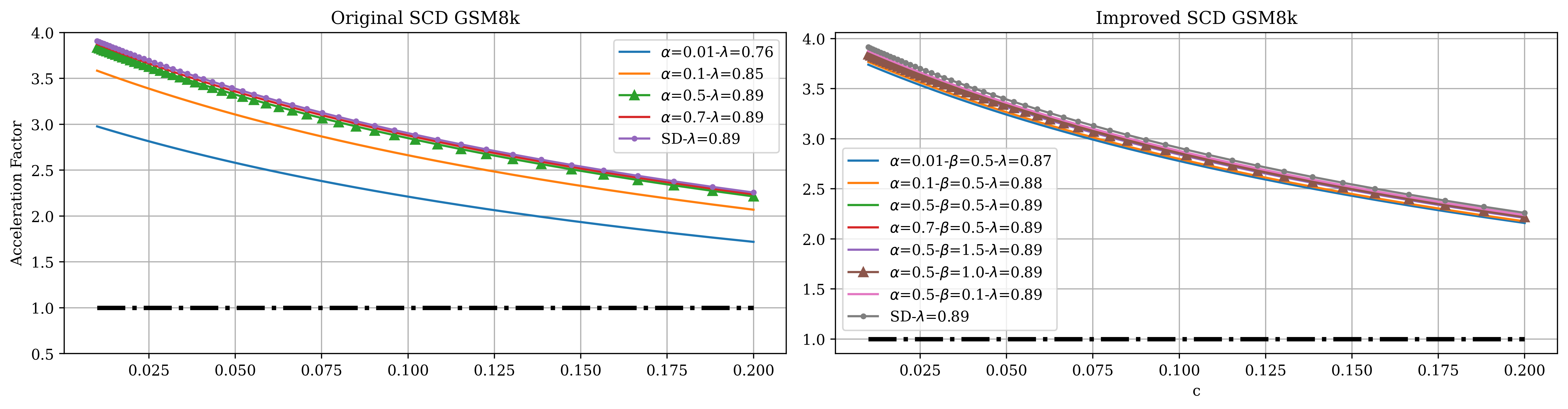}
    \caption{Hyper-parameter analysis on expected acceleration factors regarding empirical acceptance rate $\lambda$.
    The best hyper-parameter settings as in \Cref{tab:main_table} are the lines marked with triangles.
    }
    \label{app:fig:analysis_hyperparam_gsm8k}
\end{figure*}

\begin{figure*}[t]
    \centering
    \includegraphics[width=\linewidth]{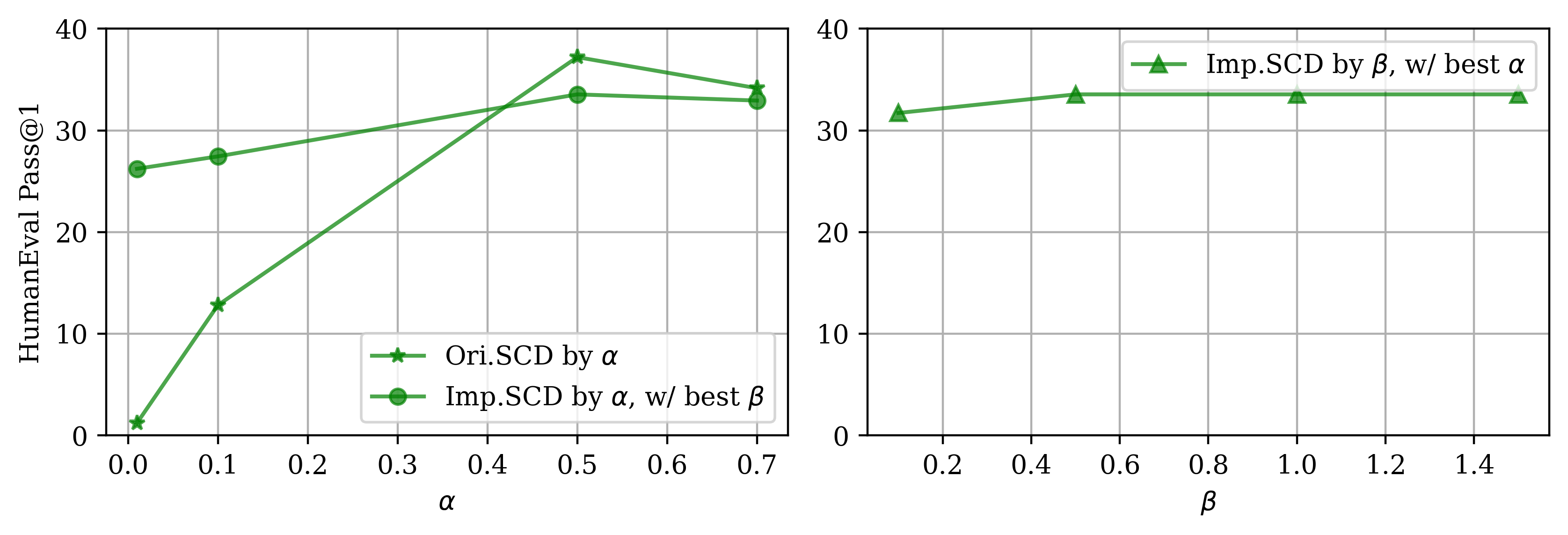}
    \caption{Performance sensitivity regarding $\alpha$ and $\beta$.
    }
    \label{app:fig:analysis_performance_final_humaneval}
\end{figure*}

\end{document}